\documentclass{article}
\usepackage{geometry}
\usepackage{graphicx}
\usepackage{amssymb}
\usepackage{amsmath}
\usepackage{amsfonts}

\usepackage{algorithm}
\usepackage{algpseudocode}

\def\address#1{\expandafter\def\expandafter\@aabuffer\expandafter
	{\@aabuffer{\affiliationfont{#1}}\relax\par
	\vspace*{13pt}}}
\def\keywords#1{\par
	\vspace*{8pt}
	{\footnotesize{\leftskip18pt\rightskip\leftskip
	\noindent{Keywords}\/:\ #1\par}}\par}
\def\ccode#1{\par
	\vspace*{8pt}
	{\footnotesize{\leftskip18pt\rightskip\leftskip
	\noindent #1\par}}\par}

\newenvironment{proof}[1][Proof]{\textbf{#1.} }{\ \rule{0.5em}{0.5em}}

\newtheorem{theorem}{Theorem}

\newtheorem{corollary}[theorem]{Corollary}

\newtheorem{definition}[theorem]{Definition}

\newtheorem{lemma}[theorem]{Lemma}

\newtheorem{proposition}[theorem]{Proposition}

\newcommand{\diag}{\operatorname{diag}}
\newcommand{\Span}{\operatorname{span}}
\newcommand{\Tr}{\operatorname{Tr}}
\newcommand{\argmin}{\operatornamewithlimits{arg \, min}}

\begin{document}


\title{Mathematical Analysis on Out-of-Sample Extensions}%
\author{Jianzhong Wang}
\address{Department of Mathematics and Statistics\\ Sam Houston State University, TX 77341, USA\\%
jzwang@shsu.edu}
\maketitle

\begin{abstract}
Let $X=\mathbf{X}\cup\mathbf{Z}$ be a data set in $\mathbb{R}^D$, where $\mathbf{X}$ is the training set and $\mathbf{Z}$ is the test one. Many unsupervised learning algorithms based on kernel methods have been developed to provide dimensionality reduction (DR) embedding for a given training set $\Phi: \mathbf{X} \to \mathbb{R}^d$ ( $d\ll D$) that maps the high-dimensional data $\mathbf{X}$ to its low-dimensional feature representation $\mathbf{Y}=\Phi(\mathbf{X})$. However, these algorithms do not straightforwardly produce DR of the test set $\mathbf{Z}$. An out-of-sample extension method provides DR of $\mathbf{Z}$ using an extension of the existent embedding $\Phi$, instead of re-computing the DR embedding for the whole set $X$. Among various out-of-sample DR extension methods, those based on Nystr\"{o}m approximation are very attractive. Many papers have developed such out-of-extension algorithms and shown their validity by numerical experiments. However, the mathematical theory for the DR extension still need further consideration.  Utilizing the reproducing kernel Hilbert space (RKHS) theory, this paper develops a preliminary mathematical analysis on the out-of-sample DR extension operators. It treats an out-of-sample DR extension operator as an extension of the identity on the RKHS defined on $\mathbf{X}$. Then the Nystr\"{o}m-type DR extension turns out to be an orthogonal projection. In the paper, we also present the conditions for the exact DR extension and give the estimate for the error of the extension.
\end{abstract}



\section{Introduction}\label{sec1}
Recently, high-dimensional data---speech signals, images, videos, text documents, hand-writing letters and numbers, fingerprints, etc., become more and more popular in our real-life and in scientific and technological areas. Due to the curse of dimensionality\cite{Bellman61,ScottT83}, directly analyzing and processing high-dimensional data are very difficult and often infeasible. Therefore, \emph{dimensionality reduction} (DR)\cite{LeeV07,Wang12} is critical in high-dimensional data processing. The purpose of DR is to find a low-dimensional representation for a given high-dimensional data set, whose main features are preserved, so that the data processing can be carried out on the low-dimensional data set. There exist many DR methods in literature. The famous linear method is \emph{principle component analysis} (PCA)\cite{Jolliffe86}. From the view point of geometry, PCA is only effective when the data set nearly lies on a hyperplane. If the data set resides on a (nonlinear) manifold instead, PCA cannot effectively catch the data features. Then non-linear DR methods are employed. Among the nonlinear methods, the kernel methods (also called spectral methods) are well-developed and widely adopted. There are two types of DR kernels:  \emph{Gramian type} and \emph{Laplacian type}. The entries of a Gramian type kernel measure the similarity between data points. The eigenfunctions corresponding to large eigenvalues of such a kernel represent the main features of the data so that they span the feature space. Such a method essentially performs a PCA on the Gramian-type kernel. Therefore, in literature, it is often named as kernel PCA\cite{Schoel_98a}.  Isomaps\cite{BalasubramanianS02}, Diffusion Maps\cite{CoifmanL06}, and Spectral Clustering\cite{NgJW01,ShiM00} are belong to this category.

The DR methods of Laplacian type include locally linear embedding (LLE)\cite{RoweisS00}, local tangent space alignment (LTSA)\cite{ZhangZ04}, Hessian local linear embedding (HLLE)\cite{DonohoG03}, and Laplacian eigenmaps (Lmaps)\cite{BelkinN03a}. In these methods, the feature spaces turn to be the (numerical) null spaces of the kernels. All of kernels in these methods are normalized such that $1$ is their largest eigenvalue. It is clear that, if $K$ is a Laplacian-type kernel, then $I-K$ is a Gramian-type one. Hence, we can focus our study of out-of-example DR extension on the Gramian-type DR kernels only.

Out-of-example DR extension method finds DR of new (test) data set based on the DR of the training one.
Let the data set $\mathbf{X}\subset\mathbb{R}^D$ be a previously given data set, which is treated as a training set. Assume that a DR method produces a DR embedding $\mathfrak{F}$, which embeds $\mathbf{X}$ into $\mathbb{R}^d$: $\mathfrak{F}(\mathbf{X})=\mathbf{Y}\subset \mathbb{R}^d, d\ll D,$ so that $\mathbf{Y}$ provides a low-dimensional representation of $\mathbf{X}$. Assume that $\mathbf{Z}\subset \mathbb{R}^D$ is a new data set, which has the similar features as $\mathbf{X}$. An important task is to find DR for the union $X=\mathbf{X}\cup\mathbf{Z}$. To do this job, we have two options: (1) Retraining the whole set $X$ using the same DR method. Let $\mathfrak{G}$ be the new DR embedding on $X$. Then $\mathfrak{G}(X)$ gives a new DR for both sets $\mathbf{X}$ and $\mathbf{Z}$, where $\mathfrak{G}(\mathbf{X})$ may be slightly different from $\mathfrak{F}(\mathbf{X})$. (2) Extending the existent DR embedding $\mathfrak{F}$ to the new set $\mathbf{Z}$. Then, without changing the DR of $\mathbf{X}$,  $\mathfrak{F}(\mathbf{Z})$ provides a DR of $\mathbf{Z}$. We call the second option an \emph{out-of-sample DR extension}.

Why is the out-of-example DR extension technique acceptable? From the view point of machine learning, $\mathfrak{F}$ provides a feature mapping from $\mathbf{X}$ to its feature space $S$. If the feature space has finite dimension $d$, we write $\mathfrak{F}=[f_1,\cdots, f_d]$, in which $f_j$ is the $j$-th feature function on $\mathbf{X}$ and $S$ is spanned by $\mathfrak{F}$. In applications, the new data set $\mathbf{Z}$ should have the similar features as $\mathbf{X}$. Therefore, it should be embedded into the same feature space $S$. To find the feature representation of $\mathbf{Z}$, it is natural to extend the feature mapping $\mathfrak{F}$ from $\mathbf{X}$ to $\mathbf{Z}$.
Geometrically, since the data set $\mathbf{X}$ has $d$ main features, we may assume it resides on a $d$-dimensional manifold $\mathcal{M}\subset \mathbb{R}^D, (d\ll D)$. Then, the DR embedding $\mathfrak{F}$ maps each point $\mathbf{x}\in \mathbf{X}$ to its manifold coordinates $\mathfrak{F}(\mathbf{x})$. Because the new data set $\mathbf{Z}$ has the similar features as $\mathbf{X}$, it should nearly reside on the manifold $\mathcal{M}$ too. Thus, the extension of the mapping $\mathfrak{F}$ on $\mathbf{Z}$ naturally provides the manifold coordinates for $\mathbf{Z}$.

In most of real-world applications, the data set $X$ has a large cardinality. It is unpractical to make a DR extension using the option (1), particularly, when the new data come as a time-stream, the retraining is very time-consuming and often infeasible. Therefore out-of-sample DR extension is an effective and feasible technique for computing DR of new data sets.

It is clear that each out-of-sample DR extension algorithm is associated with the corresponding DR method. People usually use the name of DR method to mention the out-of-sample extension algorithm, such as PCA out-of-sample DR extension and so on.

Lots of papers have introduced various out-of-example extension algorithms (see \cite{AizenbudBA15,BengioPVDRO04,CoifmanL06b} and their references). However, the mathematical theory on out-of-example extension is not studied sufficiently. Most of engineering papers only provide the algorithms without mathematical analysis.

The main purpose of this paper is to establish a preliminary mathematical theory on the out-of-sample DR extension based on the kernel methods. We analyze the extension in the framework of reproducing kernel Hilbert space (RKHS). Let the kernel of a DR embedding $\mathfrak{F}$ on the training set $\mathbf{X}$ be denoted by $\mathbf{k}$, which produces the RKHS $H_\mathbf{k}$. Let the kernel of an extension of $\mathfrak{F}$ on the set $X=\mathbf{X}\cup \mathbf{Z}$ be denoted by ${\hat k}$, which produces the RKHS $H_{\hat k}$. Let the DR kernel on $X$ be denoted by $k$, which produces the RKHS $H_k$. In the paper, we study the properties of these spaces and the relations between them.  Among all of out-of-sample DR extensions we are particularly interested in the Nystr\"{o}m-type extension, which will be discussed in details.

The paper is organizes as follows: In Section 2, we establish the mathematical model and theory for out-of-example DR extension in the RKHS framework. In Section 3, we study the out-of-sample extension of kernel PCA. In Section 4, we study the extension errors. In section 5, we discuss how to generalize the results in Section 3 to the DR extensions associated with Diffusion Maps and Spectral Clustering.

\section{Preliminary}
We first establish the mathematical model of the out-of-example DR extension for Gramian-type DR kernels. As mentioned in Introduction, let the data set $\mathbf{X} \subset\mathbb{R}^D$ be a previously given training set and $\mathbf{Z}\subset\mathbb{R}^D$ be the test set. Their union is the whole set  $X=\mathbf{X}\cup \mathbf{Z}$. For clarity, a point in the set $X$ will be written as $x$, and a point in the set $\mathbf{X}$, when it needs to stress, is written in the bold font $\mathbf{x}$. Assume also that a finite (positive) measure $\mu(\mathbf{x})$ is defined on the set $\mathbf{X}$. Let $L^2(\mathbf{X}, \mu)$ be the (real) Hilbert space defined on $\mathbf{X}$ equipped with the inner product
\begin{equation*}
\langle f,g\rangle_{L^2(\mathbf{X},\mu)} =\int_\mathbf{X} f(\mathbf{x})g(\mathbf{x})d\mu(\mathbf{x}),\quad f,g\in L^2(\mathbf{X}, \mu).
\end{equation*}
Then, $\|f\|_{L^2(\mathbf{X},\mu)}=\sqrt{\langle f,f\rangle_{L^2(\mathbf{X},\mu)}}$. Later, we will abbreviate $L^2(\mathbf{X}, \mu)$ to $L^2(\mathbf{X})$ if the measure $\mu$ is not stressed. For convenience, the extension of $\mu$ on $X$ is still denoted by itself, which is assumed to be positive and finite on $X$ too. In the similar way, we define the space $L^2(X,\mu)$ and the inner product $\langle f,g\rangle_{L^2(X)}$. We also abbreviate $L^2(X,\mu)$ to $L^2(X)$ if no confusion arises.
\begin{definition}\label{def1.1}
A function $k:X^2\to\mathbb{R}$ is called a Mercer's kernel if it satisfies the following conditions:
\begin{enumerate}
\item $k$ is symmetric: $k(x,y)=k(y,x)$;
\item $k$ is positive semi-definite;
\item $k$ is bounded on $X^2$, that is, there is an $M>0$ such that  $|k(x,y)|\le M , (x,y)\in X^2$.
\end{enumerate}
\end{definition}
In this paper, we only consider Mercer's kernels. Hence, the term \emph{kernel} will stand for Mercer's kernel. The kernel distance (associated with $k$) between two points $x,y\in X$ is defined by
\begin{equation}\label{eq1.1}
d_k(x,y)=\sqrt{k(x,x)+k(y,y)-2k(x,y)}.
\end{equation}
Recall that the kernel $e(x,y)=x^Ty$ is the Gramian of the data set $X$. Hence, $d_e(x,y)=d_2(x,y)$, where $d_2(\cdot,\cdot)$ denotes the Euclidean distance between $x$ and $y$ in $\mathbb{R}^D$.

A kernel $k$ defines a compact operator $K$ on $L^2(X)$:
\begin{equation*}
(K f)(x) =\int_X k(x,y)f(y) d\mu(y),\quad f\in L^2(X),
\end{equation*}
which has the spectral decomposition
\begin{equation}\label{eq1.2}
k(x,y)=\sum_{j=1}^m \lambda_j v_j(x)v_j(y),\quad 1\le m\le \infty
\end{equation}
where the eigenvalues are arranged decreasingly, $\lambda_1\ge \cdots\ge \lambda_m>0$, and the eigenfunctions $[v_1,\cdots, v_m]$ are normalized to satisfy
\begin{equation*}
\langle v_i,v_j\rangle_{L^2(X)} =\delta_{i,j}.
\end{equation*}
In \eqref{eq1.2}, $m$  must be finite if the cardinality $|X|<\infty$.
By \eqref{eq1.2}, we have $Kv_j=\lambda_j v_j$, i.e.,
\begin{equation}\label{eq1.2c}
v_j=\frac{1}{\lambda_j}\int_X k(x,y)v_j(y)d\mu(y).
\end{equation}
 The kernel $k$ defines a RKHS $H$, in which the inner product satisfies\cite{Aronszajn50}
\begin{equation*}
\langle f(\cdot), k(x,\cdot) \rangle_{H}=f(x),\quad f\in H, x\in X.
\end{equation*}
For $f,g\in L^2(X)$, with $f=\sum_j c_jv_j$ and $g=\sum_j d_jv_j$, we have
\begin{equation*}
\langle f, g\rangle_{L^2(X)}=\sum_j c_j d_j.
\end{equation*}
If $f$ and $g$ are also in $H$, by Mercer's theorem,
\begin{equation}\label{eq1.2a}
\langle f, g\rangle_{H}=\sum_j \frac{c_j d_j}{\lambda_j}.
\end{equation}
Therefore, when $m<\infty$, the norm $\|\cdot\|_{L^2(X)}$ is equivalent to the norm $\|\cdot\|_{H}$; when $m=\infty$, $\|f\|_{L^2(X)}\le \lambda_1 \|f\|_{H}$ so that $f\in H$ implies $f\in L^2(X)$, but the reverse is not true.

Define $\phi_j=\sqrt{\lambda_j}v_j$, Then $\langle \phi_i,\phi_j\rangle_{H} =\delta_{i,j}$, which implies that the set $\{\phi_1,\cdots, \phi_m\}$ is an o.n. basis of $H$ and
\begin{equation} \label{eq1.2b}
k(x,y)=\sum_{j=1}^m \phi_j(x)\phi_j(y).
\end{equation}
By \eqref{eq1.2c}, we have
\begin{equation}\label{eq1.2d}
\phi_j(x)=\frac{1}{\lambda_j}\int_X k(x,y)\phi_j(y)d\mu(y).
\end{equation}
\begin{definition}\label{def1.2}
The mapping $\Phi: X\to\mathbb{R}^{m}: \Phi(x)=[\phi_1(x),\cdots,\phi_m(x)]^T$ is called the feature mapping (or the DR mapping) associated with the kernel $k$, the function $\phi_j$ is called a feature function on $X$, the space spanned by all feature functions is called a feature space, and the data set $\Phi(X)\subset\mathbb{R}^m$ is called a DR of $X$.
\end{definition}
\begin{proposition}\label{pro1.1}
The feature mapping $\Phi$ preserves the kernel distance:
\begin{equation*}
d_2(\Phi(x),\Phi(y))=d_k(x,y).
\end{equation*}
\end{proposition}
\begin{proof}
\begin{equation*}
d^2_2(\Phi(x),\Phi(y))=\sum_{j=1}^m (\phi_j(x)-\phi_j(y))^2=\sum_{j=1}^m ({\phi}^2_j(x)+{\phi}^2_j(y)-2\phi(x)\phi(y)=d^2_k(x,y).
\end{equation*}
The proposition is proved.
\end{proof}

The orthogonality in a RKHS usually is different from that in a $L^2$ space. We have the following:
\begin{proposition}\label{pro1.2}
Let $H$ be an $m$-dimensional KRHS with the kernel $a(x,y)$, which has the Cholesky decomposition $a(x,y)=\sum_{j=1}^m a_j(x)a_j(y)$. Then the set $A=\{a_1,\cdots,a_m\}\subset H$ is an o.n. basis of $H$ if and only if the set $A$ is linearly independent.
\end{proposition}
\begin{proof}
We first assume that $A$ is linearly independent. Since $H$ is a KRHS with the kernel $a(x,y)=\sum_{j=1}^m a_j(x)a_j(y)$, we have $H=\Span(a_1,\cdots, a_m)$  and
\begin{equation*}
a_i(x)=\langle a_i(\cdot),\sum_{j=1}a_j(x)a_j(\cdot) \rangle_H=\sum_{j=1}^m a_j(x) \langle a_i, a_j\rangle_H.
\end{equation*}
By the linear independence of $A$, $\langle a_i, a_j\rangle_H=\delta_{i,j}$. Hence, $\{a_1,\cdots,a_m\}$ is an o.n. basis of $H$.  The proof of the reverse part is trivial.
\end{proof}

Let ${\tilde X}$ be a subset of $X$. Denote by ${\tilde k}(x,y)$ the restriction of $k(x,y)$ on ${\tilde X}^2$: ${\tilde k}(x,y)=k(x,y), (x,y)\in {\tilde X}^2$. It is clear that ${\tilde k}$ is also a kernel. Let the restriction of $f\in H$ on ${\tilde X}$ be denoted by ${\tilde f}$. Then all of these functions form a subspace of $H$ denoted by ${\tilde H}=\{{\tilde f};\ {\tilde f}(x)=f(x), f\in H,x\in {\tilde X}\}$. It is clear that ${\tilde k}$ is a reproducing kernel of ${\tilde H}$ with the inner product $\langle {\tilde f},{\tilde g}\rangle_{\tilde H}$. The author of \cite{Aronszajn50} proved the following:
\begin{proposition}\label{pro1.2b}
For ${\tilde f}\in{\tilde H}$, we have
\begin{equation}\label{eq1.1x}
\|{\tilde f}\|_{{\tilde H}}=\min (\|f\|_{H};\ f\in H; f|_{{\tilde X}}={\tilde f}).
\end{equation}
\end{proposition}
We also need the following proposition for the kernel decomposition (see \cite{Aronszajn50}).
\begin{proposition}\label{pro1.3}
Assume that the kernel $a(x,y)$ of RKHS $H$ is the sum $a(x,y)=a_1(x,y)+a_2(x,y)$, where $a_1(x,y)$ and $a_2(x,y)$ both are Mercer's kernels. Let $H_1$ and $H_2$ be the RKHSs associated with $a_1$ and $a_2$, respectively. Then for any $f_1\in H_1$ and $f_2\in H_2$,  $f=f_1+f_2\in H$ and
\begin{equation}\label{eq1.y}
\|f\|^2_H=\min \left(\|f_1\|^2_{H_1}+\|f_2\|_{H_2}^2\right),
\end{equation}
where the minimum is taken over all of the decompositions $f=f_1+f_2$ with $f_1\in H_1, f_2\in H_2$. Furthermore,
\begin{equation}\label{eq1.z}
\|f\|^2_H=\|f_1\|^2_{H_1}+\|f_2\|_{H_2}^2.
\end{equation}
if and only if $H_1\cap H_2=\{0\}$.
\end{proposition}
We will write $k\gg k_1$ if the difference $k(x,y)-k_1(x,y)=k_2(x,y)$ is a kernel. By \eqref{eq1.y}, we have $\|f_1\|_H\le \|f_1\|_{H_1}$ for any $f_1\in H_1$. Furthermore, if $H_1\cap H_2=\{0\}$, then by \eqref{eq1.z} we have $\|f_i\|_H=\|f_i\|_{H_i}$ and  $\langle f_1, f_2\rangle_H=0$ for any $f_i\in H_i, i=1,2$, which yields
\begin{equation*}
H=H_1\oplus H_2,\quad  H_1\perp H_2.
\end{equation*}

\section{Out-of-Example for Kernel PCA}
We now introduce the RKHS framework for kernel PCA out-of-example extension. Recall that the RKHS $H$ is created by the kernel $k$, which has the decomposition \eqref{eq1.2}, or equivalently, \eqref{eq1.2b}. For convenience, in this section we assume $\dim(H)=m<\infty$, although the discussion can also be applied to the case of $m=\infty$.
By Definition~\ref{def1.2}, the mapping $\Phi: X\to\mathbb{R}^{m}$ is a DR mapping associated with the kernel $k$ and the set $\Phi(X)$ is a DR of $X$.

We make the similar discussion on the training set $\mathbf{X}$.
We denote by $\mathbf{k}$ the restriction of $k$ on $\mathbf{X}^2$:
\begin{equation}\label{eq2.1}
\mathbf{k}(\mathbf{x},\mathbf{y})= k(\mathbf{x},\mathbf{y}),\quad (\mathbf{x},\mathbf{y})\in\mathbf{X}^2.
\end{equation}
The condition $|k(x,y)|\le M$ on $X^2$ implies that  $|\mathbf{k}(x,y)|\le M$ on $\mathbf{X}^2$. Denote by $\mathbf{H}$ the RKHS associated with $\mathbf{k}$ on $\mathbf{X}$. Then
\begin{equation}\label{eq2.5}
\langle \mathbf{f}, \mathbf{k}(\mathbf{x},\cdot)\rangle_\mathbf{H} = \mathbf{f}(\mathbf{x}),\quad \mathbf{f}\in\mathbf{H}.
\end{equation}
Assume the spectral decomposition of $\mathbf{k}(\mathbf{x},\mathbf{y})$ is
\begin{equation}\label{eq2.6}
\mathbf{k}(\mathbf{x},\mathbf{y}) =\sum_{j=1}^d \sigma_j \mathbf{v}_j(\mathbf{x})\mathbf{v}_j(\mathbf{y})=V^T(\mathbf{x})\Sigma V(\mathbf{y})
\end{equation}
where $\Sigma=\diag(\sigma_1,\cdots,\sigma_d)$ with $\sigma_1\ge \sigma_2\ge \cdots\ge \sigma_d>0,$ ($d\le m$), $V(\mathbf{x})=[\mathbf{v}_1(\mathbf{x}),\cdots, \mathbf{v}_d(\mathbf{x})]^T$ with $\langle \mathbf{v}_i,\mathbf{v}_j\rangle_{L^2(\mathbf{X})} =\delta_{i,j}$. Writing $\psi_j=\sqrt{\sigma_j}\mathbf{v}_j$, we have
\begin{equation}\label{eq2.7}
\mathbf{k}(\mathbf{x},\mathbf{y})=\sum_{j=1}^d  \psi_j(\mathbf{x}) \psi_j(\mathbf{y})
\end{equation}
so that the set $\{\psi_1,\cdots, \psi_d\}$ forms an o.n. basis of $\mathbf{H}$, which defines the DR mapping: ${\Psi}:\mathbf{X} \to \mathbb{R}^d$ and
\begin{equation*}
{\Psi}(\mathbf{x})=[\psi_1(\mathbf{x}),\cdots, \psi_d(\mathbf{x})]^T
\end{equation*}
provides a DR of the data set $\mathbf{X}$.\\
To study out-of-sample DR extensions, we give the following:
\begin{definition}\label{def2.1}
Let $\mathbf{I}$ be the identity operator on $\mathbf{H}$ and $\mathbf{E}: \mathbf{H}\to H $ the continuous extension of $\mathbf{I}$ on $H$ such that, for each $\mathbf{f}\in \mathbf{H}$, $f=\mathbf{E}(\mathbf{f})$ satisfies
\begin{equation}\label{eq2.7a}
f(\mathbf{x})= \mathbf{f}(\mathbf{x}),\quad \mathbf{x}\in\mathbf{X}.
\end{equation}
Then we call $f$ an out-of-sample extension of $\mathbf{f}\in\mathbf{H}$. Particularly, we call
$$\mathbf{E}(\Psi)(X)=[\mathbf{E}(\psi_1)(X),\cdots, \mathbf{E}(\psi_d)(X)]^T$$
an out-of-sample DR of $X$ (associated with $\mathbf{E}$), and $\mathbf{E}(\Psi)(x)$ the out-of-sample DR of the sample $x$. We say that the out-of-example DR is exact if $\mathbf{E}(\Psi)(X)=\Phi(X)$, where $\Phi(X)$ is defined in Definition~\ref{def1.2}.
\end{definition}
We now give an integral representation of the identity $\mathbf{I}$.
\begin{lemma}\label{lem2.1}
Let the reproducing kernel $\mathbf{k}$ of $\mathbf{H}$ be given by \eqref{eq2.7}. Then
\begin{equation}\label{eq2.9}
\psi_j(\mathbf{x})=\frac{1}{\sigma_j}\int_\mathbf{X} \mathbf{k}(\mathbf{x},\mathbf{y})\psi_j(\mathbf{y})d\mu(\mathbf{y})
\end{equation}
so that the identity operator $\mathbf{I}$ on $\mathbf{H}$ has the following integral representation:
\begin{equation}\label{eq2.8}
\mathbf{f}(\mathbf{x})=\mathbf{I} (\mathbf{f})(\mathbf{x})=\sum_{j=1}^d c_j\frac{1}{\sigma_j}\int_\mathbf{X} \mathbf{k}(\mathbf{x},\mathbf{y})\psi_j(\mathbf{y})d\mu(\mathbf{y}), \quad \mathbf{f}=\sum_{j=1}^d c_j \psi_j\in \mathbf{H},
\end{equation}
and, for any $\mathbf{g}\in\mathbf{H}$,
\begin{equation}\label{eq2.10}
\langle \mathbf{g}, \mathbf{f}\rangle_\mathbf{H}=\sum_{j=1}^d c_j\frac{1}{\sigma_j}\int_\mathbf{X} \mathbf{g}(\mathbf{y}) \psi_j(\mathbf{y})d\mu(\mathbf{y}).
\end{equation}
\end{lemma}
\begin{proof}
By \eqref{eq2.6}, we have \eqref{eq2.9}. Then, applying \eqref{eq2.9}, we obtain \eqref{eq2.8}. Finally, for $\mathbf{g}\in \mathbf{H}$,  by $\langle \mathbf{g}(\cdot), \mathbf{k}(\mathbf{x}, \cdot)\rangle_\mathbf{H}=\mathbf{g}(\mathbf{x})$ and \eqref{eq2.8}, we get \eqref{eq2.10}.
\end{proof}

Note that the extensions of $\mathbf{I}$ are not unique. Among them, an important one is the Nystr\"{o}m-type extension, which is defined as follows.
\begin{definition}\label{def2.2}
Let the extension operator $\mathbf{T}$ be defined by
\begin{equation}\label{eq2.71}
\mathbf{T}(\mathbf{f})(x)=\sum_{j=1}^d c_j\frac{1}{\sigma_j}\int_\mathbf{X} k(x,\mathbf{y})\psi_j(\mathbf{y})d\mu(\mathbf{y}), \quad \mathbf{f}=\sum_{j=1}^d c_j \psi_j \in \mathbf{H}.
\end{equation}
We call ${\hat f}=\mathbf{T}(\mathbf{f})\in H$ the Nystr\"{o}m-type extension of $\mathbf{f}$.
Write $\hat\psi_j=\mathbf{T}(\psi_j)$ and $\hat\Psi=[\hat\psi_1,\cdots,\hat\psi_d]^T$. We call $\hat\Psi(X)$ the Nystr\"{o}m-type out-of-sample DR of $X$, and call
$\hat\Psi(x)$ the Nystr\"{o}m-type out-of-sample DR of the sample $x$.
\end{definition}
By \eqref{eq2.71}, the Nystr\"{o}m-type out-of-sample DR extension has the following representation:
\begin{equation}\label{eq2.81}
\hat\psi_j(x)=\frac{1}{\sigma_j}\int_\mathbf{X}  k(x,\mathbf{y})\psi_j(\mathbf{y})d\mu(\mathbf{y}),\quad  x\in X,\quad 1\le j\le d.
\end{equation}
The Nystr\"{o}m-type out-of-sample DR extensions of all functions in $\mathbf{H}$ form a subspace ${\hat H}$ (of $H$):
\begin{equation*}
{\hat H} =\mathbf{T}(\mathbf{H})= \{{\hat f}\in H;\ {\hat f}=\mathbf{T}(\mathbf{f}), \mathbf{f}\in\mathbf{H}\}.
\end{equation*}
Then
\begin{equation}\label{eq2.11}
{\hat k}(x,y)=\sum_{j=1}^d \hat\psi_j(x)\hat\psi_j(y), \quad (x,y)\in X^2,
\end{equation}
is a reproducing kernel for the RKHS ${\hat H}$. By Proposition~\ref{pro1.2},  $\{\hat\psi_1,\cdots,\hat\psi_d\}$ is an o.n. basis of ${\hat H}$.

We denote by  $\mathbf{T}^* : H\to \mathbf{H}$ the adjoint operator of $\mathbf{T}$. Let $\mathbf{f}\in\mathbf{H}, g\in H$. By \eqref{eq2.81}, we have
\begin{equation*}
\langle g, \mathbf{T}(\mathbf{f})\rangle_H=\sum_{j=1}^d c_j\frac{1}{\sigma_j}\int_\mathbf{X} \langle g,k(\cdot, \mathbf{y})\rangle_H \psi_j(\mathbf{y})d\mu(\mathbf{y})=\sum_{j=1}^d c_j\frac{1}{\sigma_j}\int_\mathbf{X}  g(\mathbf{y})\psi_j(\mathbf{y})d\mu(\mathbf{y})=\langle \mathbf{g}, \mathbf{f}\rangle_\mathbf{H},
\end{equation*}
where $\mathbf{g}$ is the restriction of $g$ on $\mathbf{X}$. Hence, $\mathbf{T}^*$ is the restriction operator from $X$ to $\mathbf{H}$:
\begin{equation*}
\mathbf{T}^*(f)(\mathbf{x})= f(\mathbf{x}),\quad f\in H, \mathbf{x}\in \mathbf{X}.
\end{equation*}
Then, the operator $\mathbf{T}^*\mathbf{T}: \mathbf{H}\to \mathbf{H}$ is the identity on $\mathbf{H}$: $\mathbf{T}^*\mathbf{T} = \mathbf{I}$, while the operator $P=\mathbf{T}\mathbf{T}^*$ is an orthogonal projection from $H$ to ${\hat H}$. We now prove the following:
\begin{theorem}\label{th2.3}
The function
\begin{equation*}
k_0(x,y)= k(x,y)-{\hat k}(x,y)
\end{equation*}
is a Mercer's kernel on $X$. Let $H_0$ be the RKHS associated with $k_0$. Then
\begin{equation}\label{eq2.13}
H_0=(I-P)(H)=\{f\in H;\ \mathbf{T}^*(f)=0\},
\end{equation}
so that
\begin{equation}\label{eq2.14}
H={\hat H}\oplus H_0, \quad {\hat H}\perp H_0.
\end{equation}
Consequently, the out-of-sample extension given by \eqref{eq2.8} is exact if and only if $\dim(H_0)=0$, or equivalently, $k(x,y)={\hat k}(x,y)$.
\end{theorem}
\begin{proof} Since $\mathbf{T}$ is one-to-one and onto from $\mathbf{H}$ to ${\hat H}$, $\dim({\hat H})=\dim(\mathbf{H})=d$. Because $P$ is an orthogonal projection from $H$ to ${\hat H}$, for any $f\in H$, $\|f\|_H \ge \|P(f)\|_H$, which yields $k\gg {\hat k}$ and therefore $k_0$ is a Mercer's kernel and so that \eqref{eq2.13} and \eqref{eq2.14} hold. It is clear that $\dim(H_0)=\dim(H)-\dim({\hat H})=m-d$.
If $\dim(H_0)=0$, we must have $k(x,y)={\hat k}(x,y)$. Hence, the extension \eqref{eq2.8} is exact. On the other hand, it the extension given by \eqref{eq2.8} is exact, we must have $\dim(H_0)=0$ and $k(x,y)={\hat k}(x,y)$.
\end{proof}

In case that $\dim(\mathbf{H})<\dim(H)$, the out-of-example extensions are not unique. In fact, there are infinitely many such extensions. The following corollary confirms that the Nystr\"{o}m-type extension given by \eqref{eq2.8} achieves the  minimal $H$-norm.
\begin{corollary}\label{cor2.1}
For $\mathbf{f}\in \mathbf{H}$, define $H_\mathbf{f}=\{f\in H;\ \mathbf{T}^* f = \mathbf{f}\}$. Then $\mathbf{T}(\mathbf{f})$ achieves the minimal $H$-norm in the set $H_\mathbf{f}$:
\begin{equation}\label{eq2.11b}
\|\mathbf{T}(\mathbf{f})\|_H =\argmin_{f\in H_\mathbf{f}} \|f\|_H.
\end{equation}
\end{corollary}
\begin{proof}
By \eqref{eq2.71}, $\mathbf{T}$ is an isometric mapping from $\mathbf{H}$ to ${\hat H}$. Hence, for any $\mathbf{f}\in\mathbf{H}$, $\|\mathbf{f}\|_\mathbf{H}=\|\mathbf{T} \mathbf{f}\|_H$. By \eqref{eq2.13}, for any $f\in H_\mathbf{f}$, there is a $g\in H_0$ such that $f=\mathbf{T}(\mathbf{f}) + g$. Therefore, by $\mathbf{T}(\mathbf{f})\perp g$,
$\|f\|_H =\|\mathbf{T} (\mathbf{f})\|_H +\|g\|_H\ge \|\mathbf{T} (\mathbf{f})\|_H$, which yields \eqref{eq2.11b}.
\end{proof}

When $\dim(H_0)=s>0$, the DR sets $\Phi(X)$  and $\\hat{\Psi}(X)$ usually are different. By Proposition~\ref{pro1.3}, we have the following:
\begin{corollary}\label{cor2.2}
The kernel distances associated with $K$, ${\tilde k}$ and $k_0$ satisfy the Pythagorean identity:
\begin{equation}\label{eq2.23}
d^2_k(x,y)=d^2_{\tilde k}(x,y)+d^2_{k_0}(x,y).
\end{equation}
\end{corollary}
The following corollary characterizes the reproducing kernel $k_0$ of $H_0$.
\begin{corollary}\label{cor2.3}
The reproducing kernel $k_0(x,y)=0$ for $x\in\mathbf{X}$ or $y\in\mathbf{X}$, or equivalently, $k(x,y)={\hat k}(x,y)$ for $(x,y)\in X^2\setminus \mathbf{Z}^2$.
\end{corollary}
\begin{proof}
Since $H_0=\{f\in H; \ T^*(f)=0\}$, for any $f\in H_0$, $f(\mathbf{x})=0$, if $\mathbf{x}\in \mathbf{X}$. Because $\dim(H_0)=s$, $k_0(x,y)$ has the spectral decomposition
\begin{equation}\label{eq2.22}
k_0(x,y)=\sum_{j=1}^{s} \eta_j \alpha_j(x)\alpha_j(y),
\end{equation}
where $\eta_1\ge \cdots\ge \eta_s>0$, and $\alpha_j(x)=0, \mathbf{x}\in \mathbf{X}$. Therefore, $k_0(x,y)=0$ if $x\in\mathbf{X}$ or $y\in\mathbf{X}$.
\end{proof}

Write $\xi_j(x)=\sqrt{\eta_j}\alpha_j(x)$, then,
\begin{equation}\label{eq2.22b}
k_0(x,y)=\sum_{j=1}^s \xi_j(x)\xi_j(y).
\end{equation}
so that $\{\xi_1,\cdots,\xi_s\}$  is an o.n. basis of $H_0$. By Corollary~\ref{cor2.3}, we also have the following:
\begin{equation}\label{eq2.21}
d_{k_0}(x,y)=\begin{cases}0,& (x,y)\in \mathbf{X}^2,\\
\sqrt{k_0(x,x)}, & (x,y)\in \mathbf{Z}\times \mathbf{X}\text{\ or\ }(y,x)\in \mathbf{X}\times \mathbf{Z}.
\end{cases}
\end{equation}

Denote by $P_0$ be the orthogonal projection from $H$ to $H_0$. We now give the matrix forms of the orthogonal projections $P$ and $P_0$. It is obvious that $P+P_0=I$. By $T^*(\phi_j)\in \mathbf{H}$, we may write
\begin{equation}\label{eq2.new}
T^*\phi_j=\sum_{i=1}^d c_{i,j}\psi_i,\quad c_{i,j}=\langle T^*\phi_j, \psi_i\rangle_\mathbf{H}.
\end{equation}
By the orthogonality of $\{\phi_1,\cdots,\phi_m\}$ in $H$ and
$\langle T^*\phi_j, \psi_i\rangle_\mathbf{H}=\langle \phi_j, \hat\psi_i\rangle_H $,
we also have
\begin{equation*}
\hat\psi_i=\sum_{j=1}^m c_{i,j}\phi_j.
\end{equation*}
Let $C=[c_{i,j}]_{i,j=1}^{d,m}\in \mathbb{R}^{d\times m}$.
Then, the matrix form of $P$ is given by
\begin{equation}\label{eq2.20}
P(\Phi)=C^TC \Phi,
\end{equation}
where
$
CC^T=I.
$
Recall that $k_0=k-{\hat k}$. Therefore,
\begin{equation}\label{eq2.15}
k_0(x,y)=\Phi^T(x)\Phi(y)-{\hat\Psi}^T(x)\hat\Psi(y)=\Phi^T(x)(I-C^TC)\Phi(y).
\end{equation}
Since $I-C^TC$ is an orthogonal projection matrix with rank $s=m-d$, it has the following spectral decomposition:
\begin{equation}\label{eq2.19}
I-C^T C=Q^T Q,
\end{equation}
where $Q^T=[q_1,\cdots, q_s]^T\in\mathbb{R}^{m\times s}$ satisfies
$
QQ^T=I.
$
It is clear that $P_0(f)=Q^TQ f$. That is, if $f(x)=\sum_{j=1}^m f_j\phi_j(x)$  and $P_0 f(x)=\sum_{j=1}^m {\tilde f}_j \xi_j(x)$, then,
\begin{equation*}
\tilde{F}= Q^TQF,
\end{equation*}
where $F=[f_1,\cdots,f_m]^T$ and $\tilde{F}=[{\tilde f}_1,\cdots,{\tilde f}_m]^T$.
Write $\tilde\Psi=[\hat\psi_1,\cdots,\hat\psi_d, \xi_1,\cdots,\xi_s]$ and $O=\begin{bmatrix} C\\Q\end{bmatrix}$. Then $O$ is the orthogonal transform $\Phi\rightarrow\tilde\Psi$:
\begin{equation*}
\tilde\Psi = O \Phi.
\end{equation*}
\section{Estimates of approximative errors of Nystr\"{o}m-type out-of-example extension}
We have introduced  two different versions of DR of the data set $X$: The lossless kernel PCA DR $\Phi$ and the out-of-sample DR extension $\hat\Psi$. The first one is obtained by retraining the whole data set while the second one is obtained by Nystr\"{o}m extension technique. When $s>0$, they are different. We now discuss the difference between $\hat\Psi$ and $\Phi$. Recall that the DR of a data set $X$ is not unique even though the same DR method is employed. For instance. in our case, $\hat\Psi(X)\subset\mathbb{R}^d$ is a DR of $X$. Let $U\in \mathbb{R}_{d\times d}$ be an orthogonal matrix. Then $U\hat\Psi(X)$ should essentially give the same DR because it is simply a rotation of $\hat\Psi(X)$ in $\mathbb{R}^d$, which does not change the geometric structure. To eliminate the impact caused by isometric transformation on RD sets, we estimate the difference between DRs by their eigenvalues and kernel distances. All of the estimates in this section are given in $L^2(X)$.
\subsection{Eigenvalue estimates}
To compare the eigenvalues of DR kernels $k$ and ${\hat k}$, we need the spectral decomposition of ${\hat k}(x,y)$:
\begin{equation}\label{eq3.0n}
{\hat k}(x,y)=\sum_{j=1}^d \gamma_j {g}_j(x) {g}_j(y)=G^T(x)\Gamma G(y),
\end{equation}
where $\int_X G(x)G^T(x)d\mu(x)=I$. To compare the spectra $\Lambda$ (of $k(x,y)$), $\Gamma$ (of ${\hat k}(x,y)$), and $\Sigma$ (of $\mathbf{k}(\mathbf{x},\mathbf{y})$). We give the following lemma.
\begin{lemma}\label{lem3.1}
Let $E$ be a subset of $X$ and  $\{a_1(x),\cdots, a_s(x)\}\subset L^2(E)$ a linearly independent set in $L^2(E)$. Write $A(x)=[a_1(x),\cdots,a_s(x)]^T$ and $T(x,y)=A^T(x)A(y)=\sum_{j=1}^s a_j(x)a_j(y)$, which has the spectral decomposition
\begin{equation}\label{eq3.10n}
T(x,y)=\sum_{j=1}^s b_j w_j(x)w_j(y)= W^T(x)B W(y),
\end{equation}
where $\int_E W(x)W^T(x)d\mu(x)=I$.
Let $C(x)=B^{1/2}W(x)$ and
\begin{equation*}
M=\int_E A(x)A^T(x)d\mu(x)\in \mathbb{R}^{s\times s}.
\end{equation*}
Then $M=S^TBS$, where $S\in \mathbb{R}^{s\times s}$ is an orthogonal transform from $A(x)$ to $C(x)$:  $C(x)=SA(x)$.
\end{lemma}
\begin{proof}
Let $H_T$ be the RKHS associated with the kernel $T(x,y)$. Since $T(x,y)=A^T(x)A(y)=C^T(x)C(y)$, both $C(x)$ and $A(x)$ are o.n. bases of $H_T$. Therefore, there is an orthogonal matrix $S\in \mathbb{R}^{s\times s}$ such that $C(x)=SA(x)$. We now have
\begin{equation*}
M=S^T \int_E C(x)C^T(x)d\mu(x)S=S^TB^{1/2}\int_E W(x)W^T(x)d\mu(x)B^{1/2}S=S^TBS.
\end{equation*}
The proof is completed.
\end{proof}

\begin{theorem}\label{th4.1}
Let the spectral decompositions of the kernels ${\hat k}(x,y)$ and $\mathbf{k}(\mathbf{x},\mathbf{y})$  be given by \eqref{eq3.0n} and \eqref{eq2.6}, respectively, where $\Gamma$ and $\Sigma$ are their spectra.  Let ${\hat V}(x)=\Sigma^{1/2}\hat\Psi(x)$ and
\begin{equation}\label{eq2.J}
J = \int_\mathbf{Z} \hat{V}(\mathbf{z}) {\hat{V}}^T(\mathbf{z})d\mu(\mathbf{z}).
\end{equation}
Define $\mathbf{t}(\mathbf{x},\mathbf{y})=\Psi^T(\mathbf{x})J\Psi(\mathbf{y})$, and
\begin{equation*}
\mathbf{l}(\mathbf{x},\mathbf{y})=\mathbf{k}(\mathbf{x},\mathbf{y})+\mathbf{t}(\mathbf{x},\mathbf{y}),\quad (\mathbf{x},\mathbf{y})\in \mathbf{X}^2.
\end{equation*}
Then $\mathbf{l}(\mathbf{x}, \mathbf{y})$ has the spectral decomposition
\begin{equation}\label{eq3.2n}
\mathbf{l}(\mathbf{x},\mathbf{y})= U^T(\mathbf{x})\Gamma U(\mathbf{y}),
\end{equation}
where $\int_\mathbf{X} U(\mathbf{x})U^T(\mathbf{x})d\mu(\mathbf{x})=I$.
\end{theorem}
\begin{proof}
Since $\mathbf{k}(\mathbf{x},\mathbf{y})=\Psi^T(\mathbf{x})\Psi(\mathbf{y})=\hat\Psi^T(\mathbf{x})\hat\Psi(\mathbf{y})$  for $(\mathbf{x},\mathbf{y})\in \mathbf{X}^2$, we have
\begin{equation}\label{eq3.4n}
  \int_\mathbf{X} \hat\Psi(\mathbf{x}){\hat\Psi}^T(\mathbf{x})d\mu(\mathbf{x}) = \Sigma.
\end{equation}
By \eqref{eq2.J}, we have
\begin{equation}\label{eq3.5n}
  \int_\mathbf{Z} \hat\Psi(\mathbf{z}){\hat\Psi}^T(\mathbf{z})d\mu(\mathbf{z})=\Sigma^{1/2} J \Sigma^{1/2}.
\end{equation}
By Lemma~\ref{lem3.1}, there is an orthogonal matrix $S$ such that
\begin{equation*}
\int_X \hat\Psi(x){\hat\Psi}^T(x)d\mu(x)=S^T\Gamma S.
\end{equation*}
Since
\begin{equation*}
\int_X \hat\Psi(x){\hat\Psi}^T(x)d\mu(x)=\int_\mathbf{X} \hat\Psi(\mathbf{x}){\hat\Psi}^T(\mathbf{x})d\mu(\mathbf{x})+\int_\mathbf{Z} \hat\Psi(\mathbf{z}){\hat\Psi}^T(\mathbf{z})d\mu(\mathbf{z}),
\end{equation*}
we have
$$\Gamma = S^T \Sigma^{1/2}(I+P)\Sigma^{1/2}S.$$
Recall that $\Psi(\mathbf{x})=\Sigma^{1/2}V(\mathbf{x})$, where $V(\mathbf{x})$ satisfies $\int_\mathbf{X} V(\mathbf{x})V^T(\mathbf{x})d\mu(\mathbf{x})$.  Define $U(\mathbf{x})=SV(\mathbf{x})$. Then,
\begin{equation*}
\mathbf{l}(\mathbf{x},\mathbf{y})=V^T(\mathbf{x})\Sigma V(\mathbf{y})+V^T(\mathbf{x})\Sigma^{1/2}J\Sigma^{1/2} V(\mathbf{y})
=U^T(\mathbf{x})\Gamma U(\mathbf{y}),
\end{equation*}
where
\begin{equation*}
\int_\mathbf{X} U(\mathbf{x})U^T(\mathbf{x})d\mu(\mathbf{x})=S\left(\int_\mathbf{X} V(\mathbf{x})V^T(\mathbf{x})d\mu(\mathbf{x})\right)S^T=I,
\end{equation*}
which yields \eqref{eq3.2n}.
\end{proof}

For a kernel $a(x,y)$, we denote by $\|a\|_2$ the spectral radius of $a$. Since $a$ is symmetric, positive semi-definite, $\|a\|_2$ is equal to the largest eigenvalue of $a$. We now have the following:
\begin{corollary}\label{cor3.1}
Let $\{\lambda_1,\cdots,\lambda_m\}$, $\{\gamma_1,\cdots,\gamma_d\}$, and $\{\sigma_1, \cdots, \sigma_d\}$ be the spectra of the kernels $k(x,y)$, ${\hat k}(x,y)$, and $\mathbf{k}(\mathbf{x},\mathbf{y})$, respectively. Then
\begin{equation} \label{eq3.3n}
0\le \lambda_j-\gamma_j\le \|k_0\|_2, \quad 0\le\gamma_j-\sigma_j\le \|\mathbf{t}\|_2,\quad 1\le j\le d,
\end{equation}
and
\begin{equation*}
0\le \lambda_j\le \|k_0\|_2, \quad j=d+1,\cdots,m.
\end{equation*}
\end{corollary}
\begin{proof}
Since $k_0(x,y)=k(x,y)-{\hat k}(x,y)$ and $\mathbf{t}(\mathbf{x},\mathbf{y})=\mathbf{l}(\mathbf{x},\mathbf{y})-\mathbf{k}(\mathbf{x}, \mathbf{y})$ are positive semi-definite, by the Monotonicity Theorem of Eigenvalues, we obtain the results.
\end{proof}

If $\|k_0\|_2$ is small enough, by the matrix perturbation theory, we have the following:\par
\begin{theorem}\label{th5.1new}
Let $M=Q^TQ\in\mathbb{R}^{m\times m}$ be given as in \eqref{eq2.19}, $(\lambda_i, v_i(x))$ and $(\gamma_i,g_i(x))$ be the eigenvalues and eigenfunctions of the kernels $k(x,y)$ and ${\hat k}(x,y)$ as given in \eqref{eq1.2} and \eqref{eq3.0n}, respectively. Denote by $m_{i,j}$ the $(i,j)$-entry of $M$. Then, for $1\le i\le m$,
\begin{eqnarray}
  \gamma_i &=& \lambda_i(1-m_{i,i})+O(\|k_0\|_2^2),\label{eq3.10-1} \\
  g_i(x)&=& v_i(x)-\sum_{j\ne i} \frac{m_{i,j}\sqrt{\lambda_i\lambda_j}}{\lambda_i-\lambda_j}v_j(x)+O(\|k_0\|_2^2),\label{eq3.10-2}
\end{eqnarray}
and, equivalently,
\begin{eqnarray}
  \lambda_i &=& \gamma_i(1+m_{i,i})+O(\|k_0\|_2^2),\label{eq3.10-3} \\
  v_i(x)&=& g_i(x)+\sum_{j\ne i} \frac{m_{i,j}\sqrt{\gamma_i\gamma_j}}{\gamma_i-\gamma_j}g_j(x)+O(\|k_0\|_2^2),\label{eq3.10-4}
\end{eqnarray}
where we set $\gamma_i=0$ and $g_i(x)=0$ for $i>d$.
\end{theorem}
\begin{proof}
Using the similar argument in the eigenpair first-order approximation as shown in \cite{ShmueliWA12, ShmueliSSA13}, we have
\begin{eqnarray*}
  \gamma_i &=& \lambda_i-\iint_{X^2} k_0(x,y) v_i(x)v_i(y) d\mu(x)d\mu(y)+O(\|k_0\|_2^2), \\
  g_i(x)&=& v_i(x)-\sum_{j\ne i} \frac{\iint_{X^2} k_0(x,y) v_i(x)v_j(y) d\mu(x)d\mu(y)}{\lambda_i-\lambda_j}v_j(x)+O(\|k_0\|_2^2).
\end{eqnarray*}
By \eqref{eq2.15} and \eqref{eq2.9}, we obtain \eqref{eq3.10-1} and \eqref{eq3.10-2}. In the similar way, we can obtain \eqref{eq3.10-3} and \eqref{eq3.10-4} too.
\end{proof}

\subsection{Difference of kernel distances}
The kernel distance $d_k(x,y)$ measures the Euclidean distance between the samples $x$ and $y$ in the feature space. Recall that $\Psi(X)$ and $\hat\Psi(X)$ are the DRs of the data set $X$ associated with the kernels $k(x,y)$ and ${\hat k}(x,y)$, respectively. Naturally, we measure the difference between $\Phi(X)$ and $\hat\Psi(X)$ by the following average kernel distance:
\begin{equation}\label{eq2.17}
d(\Psi,\hat\Psi)=\frac{1}{|X|} \sqrt{\iint_{X^2} \left|d^2_{k}(x,y)-d^2_{{\hat k}}(x,y)\right|d\mu(x) d\mu(y)},
\end{equation}
where $|X|=\int_X d\mu(x)$ is the volume of $X$.
By \eqref{eq2.23} and \eqref{eq2.21},  we have
\begin{equation}\label{eq2.21b}
d(\Phi,\hat\Psi)=\frac{1}{|X|}\sqrt{\iint_{\mathbf{Z}^2} d_{k_0}^2 (x,y))d\mu(x)d\mu(y)}.
\end{equation}
By the spectral decomposition of  $k_0(x,y)$  in \eqref{eq2.22}, $\Tr(k_0)=\sum_{j=1}^s \eta_j$.
We now have the following:
\begin{theorem}\label{th4.2}
\begin{equation}\label{eq3.6n}
d(\Phi,\hat\Psi)\le \frac{\sqrt{2|\mathbf{Z}|}}{|X|}\sqrt{\Tr(k_0)}.
\end{equation}
\end{theorem}
\begin{proof}
By \eqref{eq2.21b},
\begin{equation*}
d(\Phi,\hat\Psi)=\frac{1}{|X|}\sqrt{\iint_{\mathbf{Z}^2} (k_0(x,x)+k_0(y,y)-2k_0(x,y))d\mu(x)d\mu(y)},
\end{equation*}
where
\begin{equation*}
\iint_{\mathbf{Z}^2} k_0(x,x)d\mu(x)d\mu(y) =\iint_{\mathbf{Z}^2} k_0(y,y)d\mu(x)d\mu(y)= |\mathbf{Z}|\Tr(k_0)
\end{equation*}
and
\begin{equation*}
   \iint_{\mathbf{Z}^2} k_0(x,y)d\mu(x) d\mu(y) =\sum_{j=1}^s \eta_j \left(\int_\mathbf{Z} \alpha_j(\mathbf{z})d\mu(\mathbf{z}) \right)^2\ge 0.
\end{equation*}
Hence, \eqref{eq3.6n} holds.
\end{proof}

By $\Tr(k_0)\le s\|k_0\|_2$, we also have
\begin{equation}\label{eq3.6o}
d(\Phi,\hat\Psi)\le \frac{\sqrt{2s|\mathbf{Z}|}}{|X|}\sqrt{\|k_0\|_2}.
\end{equation}

\section{Out-of-Example for Diffusion Maps, Laplacian Eigenmaps, and Spectral Clustering}
Diffusion Maps and Spectral Clustering are two important examples of kernel PCA, with certain variations. They employ the same kernel. Although Laplacian Eigenmaps is not a kernel PCA method, it has a close relation with Diffusion Maps. Indeed, in the continuous model, Laplacian operator is the infinitesimal of Diffusion one. Therefore, they have the same set of eigenfunctions, and a $\lambda$ is an eigenvalue of the Laplacian if and only if $e^{-\lambda}$ is an eigenvalue of the corresponding diffusion operator.  Hence, the out-of-example extension algorithms for these three methods essentially are identical.  Let
\begin{equation*}
w(x,y)=\exp\left(-\frac{\|x-y\|^2}{\epsilon}\right),\quad (x,y)\in X^2.
\end{equation*}
Then
\begin{equation*}
S(x)=\int_X w(x,y)d\mu(y)
\end{equation*}
defines a mass density on $X$, and $S=\int_X S(x)d\mu(x)$ is the total mass of $X$.

The kernel for Diffusion Map and Spectral Clustering is the following:
\begin{equation*}
{\tilde k}(x,y)=\frac{w(x,y)}{\sqrt{S(x)S(y)}}.
\end{equation*}
Let its spectral decomposition be
\begin{equation}\label{eq3.0}
{\tilde k}(x,y)=\sum_{j=0}^{m}\mu_j\phi_j(x)\phi_j(y),
\end{equation}
where $\mu_0=1$, $\phi_0(x)=\frac{\sqrt{S(x)}}{\sqrt{S}}$, and $\mu_0\ge \mu_1\ge \cdots\ge \mu_m>0$. Writing $\tilde\phi_j = \sqrt{\mu_j}\phi_j$, we have
\begin{equation}\label{eq3.1}
{\tilde k}(x, y)=\sum_{j=0}^m \tilde\phi_j(x)\tilde\phi_j(y).
\end{equation}
Denote by $H_{\tilde k}$ the RKHP associated with ${\tilde k}$. Then the set $\{\tilde\phi_0,\cdots,\tilde\phi_m\}$ is an o.n. basis of $H_{\tilde k}$ and an orthogonal (but not o.n.) system in $L^2(X,d\mu)$. Since the function $\tilde\phi_0(x)$ is only a probability distribution function of the data set $X$, it does not present any feature of data. Therefore, we do not count it in DR. We define
\begin{equation*}
\tilde{u}_j(x)=\tilde\phi_0(x)\tilde\phi_j(x),\quad \tilde{v}_j(x)=\frac{1}{\tilde\phi_0(x)}\tilde\phi_j(x),
\end{equation*}
where $\tilde{u}_0(x)={\tilde\phi}^2_0(x)=\frac{S(x)}{S}$ and $\tilde{v}_0(x)=1$.
\begin{definition}\label{def3.1}
The vector of functions
\begin{equation}\label{eq3.2a}
\tilde\Phi(x)=[\tilde\phi_1(x),\cdots,\tilde\phi_m(x)]
\end{equation}
is called a standard DR of the data set $X$ associated to ${\tilde k}$,
the vector of functions
\begin{equation}\label{eq3.b}
\tilde{U}(x)=[\tilde{u}_1(x),\cdots,\tilde{u}_m(x)]
\end{equation}
is called a weighted DR of the data set $X$ associated to ${\tilde k}$, and
the vector of functions
\begin{equation}\label{eq3.2c}
\tilde{V}(x)= [\tilde{v}_1(x),\cdots,\tilde{v}_m(x)]
\end{equation}
is called a normalized DR of the data set $X$ associated to ${\tilde k}$.
\end{definition}
We now introduce an asymmetric kernel generated by $w(x,y)$:
\begin{equation*}
    m(x,y)=\frac{1}{S(x)}w(x,y).
\end{equation*}
It is clear that $m(x,y)\ge 0$  and, for any $x\in X$,
\begin{equation}\label{eq3.3}
\int_X m(x,y)d\mu(y)=1.
\end{equation}
Hence, $m(x,y)$ defines a random walk on the data set $X$. Note that
\begin{equation} \label{eq3.4}
    m(x,y)=\tilde{u}_0(y)+ \sum_{j=1}^{m} \tilde{u}_j(y) \tilde{v}_j(x).
\end{equation}
Denote by $\mathbf{p}(t,y|x)$ the probability of the walk from $x$ to $y$ after time $t$. By \eqref{eq3.4},  $\mathbf{p}(1,y|x) =m(x,y)$. Therefore, $\tilde{V}$ is the (unit-time) transaction vector (or diffusion mapping) and $\tilde{U}$ is the vector of the feature functions of $X$.

Denote by $H_w$ the RKHS produced by the kernel $w(x,y)$.  Write
\begin{equation*}
u_j(x)=\sqrt{S}\tilde{u}_j(x)=\sqrt{S(x)}\tilde\phi_j(x).
\end{equation*}
We have
\begin{equation*}
w(x,y)=\sum_{j=0}^m {u}_j(x){u}_j(y).
\end{equation*}
Then the set $\{{u}_0(x),\cdots,{u}_m(x)\}$ is an o.n. basis of $H_w$.
Similarly, write ${v}_j(x)=\frac{1}{\sqrt{S}}\tilde{v}_j(x)=\frac{\tilde\phi_j(x)}{\sqrt{S(x)}}$ and denote by $H_a$ the RKHS produced by the kernel
\begin{equation*}
a(x,y) =\sum_{j=0}^m {v}_j(x){v}_j(y)=\frac{w(x,y)}{S(x)S(y)}.
\end{equation*}
Then, the set $\{{v}_0(x),\cdots,{v}_m(x)\}$ is an o.n. basis of $H_a$. Denote the multiplier on $H_{\tilde k}$ by
\begin{equation}\label{eq4.0}
\mathfrak{S}_S(f)= \sqrt{S(\cdot)}f(\cdot),
\end{equation}
whose inverse is $\mathfrak{S}_S^{-1}(f)=\frac{1}{\sqrt{S(\cdot)}}f(\cdot)$. Latter, if $S(x)$ in \eqref{eq4.0} is not stressed, we will simply denote $\mathfrak{S}_S$  by $\mathfrak{S}$. It is clear that $\mathfrak{S}$ is an isometric mapping from $H_{\tilde k}$ to $H_w$ and $\mathfrak{S}^{-1}$ is an isometric mapping from $H_{\tilde k}$ to $H_a$. Thus, $\mathfrak{S}^2$ is an isometric mapping from $H_a$ to $H_w$.

We now return to the discussion of the out-of-sample extension for Diffusion Maps.  Let  the spectral decomposition of the Diffusion-Map kernel is give by
\begin{equation}\label{eq4.1}
\mathbf{k}(\mathbf{x},\mathbf{y})=\frac{\mathbf{w}(\mathbf{x},\mathbf{y})}{\sqrt{\mathbf{S}(\mathbf{x})\mathbf{S}(\mathbf{y})}}
=\sum_{j=1}^{d}\lambda_j\psi_j(\mathbf{x})\psi_j(\mathbf{y})=\sum_{j=1}^{d}\tilde\psi_j(\mathbf{x})\tilde\psi_j(\mathbf{y}),\quad (\mathbf{x},\mathbf{y})\in\mathbf{X}^2,
\end{equation}
where $\mathbf{S}(\mathbf{x})=\int_\mathbf{X} \mathbf{w}(\mathbf{x},\mathbf{y})d\mu(\mathbf{y})$ and $\tilde\psi(\mathbf{x})=\sqrt{\lambda_j}\psi_j(\mathbf{x})$.
Because of ${\tilde k}(\mathbf{x},\mathbf{y})\neq \mathbf{k}(\mathbf{x},\mathbf{y})$, the out-of-sample extension algorithms for a standard kernel PCA, as developed in the previous section, cannot be directly applied for Diffusion Maps. Hence, we need to make a modification based on the following lemma:
\begin{lemma}\label{lm3.1}
Let  the spectral decomposition of $\mathbf{k}$ be given by \eqref{eq4.1}.
Write $\mathbf{u}_j(\mathbf{x})=\sqrt{\mathbf{S}(\mathbf{x})}\tilde\psi_j(\mathbf{x})$, and
$\mathbf{v}_j(\mathbf{x})=\frac{1}{\sqrt{\mathbf{S}(\mathbf{x})}}\tilde\psi_j(\mathbf{x})$.
Then
\begin{eqnarray}
\mathbf{u}_j(\mathbf{x})&=&\frac{1}{\lambda_j}\int_\mathbf{X} \mathbf{w}(\mathbf{x},\mathbf{y})\mathbf{v}_j(\mathbf{y})d\mu(\mathbf{y}), \label{eq3.4a}\\
\mathbf{v}_j(\mathbf{x})&=&\frac{1}{\lambda_j}\int_\mathbf{X} \mathbf{w}(\mathbf{x},\mathbf{y})\mathbf{u}_j(\mathbf{y})d\mu(\mathbf{y}).\label{eq3.4b}
\end{eqnarray}
\end{lemma}
\begin{proof}
We have
\begin{equation*}
\mathbf{u}_j(\mathbf{x})=\sqrt{\mathbf{S}(\mathbf{x})}\tilde\psi_j(\mathbf{x})=\frac{1}{\lambda_j}\int_\mathbf{X} \sqrt{\mathbf{S}(\mathbf{x})}\sqrt{\mathbf{S}(\mathbf{y})}\mathbf{k}(\mathbf{x},\mathbf{y})\mathbf{v}_j(\mathbf{y})d\mu(\mathbf{y}),
\end{equation*}
which yields \eqref{eq3.4a}. The proof for \eqref{eq3.4b} is similar.
\end{proof}

By Lemma~\ref{lm3.1}, for $f\in H_\mathbf{w}$, we have
\begin{equation*}
f(\mathbf{x})=\langle f, \mathbf{w}(\mathbf{x}, \cdot)\rangle_{H_\mathbf{w}}=\sum_{j=0}^d c_j\frac{1}{\lambda_j}\int_\mathbf{X} \mathbf{w}(\mathbf{x},\mathbf{y})\mathbf{v}_j(\mathbf{y})d\mu(\mathbf{y}), \quad f=\sum_{j=0}^d c_j \mathbf{u}_j\in H_\mathbf{w}
\end{equation*}
and
\begin{equation}\label{eq3.5}
\langle g, f\rangle_{H_\mathbf{w}}=\sum_{j=0}^d c_j\frac{1}{\lambda_j}\int_\mathbf{X} g(\mathbf{y}) \mathbf{v}_j(\mathbf{y})d\mu(\mathbf{y}).
\end{equation}
We now introduce the operator $\mathbf{T}: H_\mathbf{w}\to H_w$:
\begin{equation}\label{eq3.6}
(\mathbf{T} f)(x)=\sum_{j=0}^d c_j\frac{1}{\lambda_j}\int_\mathbf{X} w(x,\mathbf{y})\mathbf{v}_j(\mathbf{y})d\mu(\mathbf{y}), \quad f=\sum_{j=0}^d c_j \mathbf{u}_j\in H_\mathbf{w}.
\end{equation}
\begin{lemma}\label{lm3.2}
We have the following:
\begin{enumerate}
    \item The adjoint operator $\mathbf{T}^*$ is the restriction from $H_w$ to $H_\mathbf{w}$: For any $F\in H_w$, $\mathbf{T}^*F(\mathbf{x})=F(\mathbf{x}), \mathbf{x}\in \mathbf{X}$.
    \item $\mathbf{T}^*\mathbf{T}=I$ on $H_\mathbf{w}$.
    \item $\mathbf{P}=\mathbf{T}\mathbf{T}^*$ is an orthogonal projection from $H_w$ to its subspace $H_{\hat w} = \mathbf{P}(H_w)$. Let ${\hat u}_j=\mathbf{T}(\mathbf{u}_j)$. Then $\{{\hat u}_0,\cdots,{\hat u}_d\}$ is an o.n. basis of $H_{\hat w}$ so that ${\hat w}(x,y)=\sum_{j=0}^d {\hat u}_j(x){\hat u}_j(y)$ is a reproducing kernel of the RKHS $H_{\hat w}$.
\end{enumerate}
\end{lemma}
\begin{proof}
Let $g\in H_w$ and $\mathbf{f}\in H_\mathbf{w}$. Then
\begin{equation*}
\langle g, \mathbf{T}(\mathbf{f})\rangle_{H_w}=\sum_{j=0}^d c_j\frac{1}{\lambda_j}\int_\mathbf{X} \langle g,{\tilde k}(\cdot, \mathbf{y})\rangle_{H_w} \mathbf{v}_j(\mathbf{y})d\mu(\mathbf{y})=\sum_{j=0}^d c_j\frac{1}{\lambda_j}\int_\mathbf{X} g(\mathbf{y})\mathbf{v}_j(\mathbf{y})d\mu(\mathbf{y})=\langle \mathbf{g}, \mathbf{f}\rangle_{H_\mathbf{w}},
\end{equation*}
where $\mathbf{g}$ is the restriction of $g$ on $\mathbf{X}$. Hence, for any $F\in H_w$,  $\mathbf{T}^* F(\mathbf{x})= F(\mathbf{x}),\mathbf{x}\in \mathbf{X}$. Item 1 is proved. By the definition of $\mathbf{T}$ in \eqref{eq3.6}, $\mathbf{T} (f)(\mathbf{x})=f(\mathbf{x}), \mathbf{x}\in\mathbf{X}$. Hence, $\mathbf{T}^*\mathbf{T}=I$ on $H_\mathbf{w}$. Item 2 is proved. Finally, it is clear that the set $\{{\hat u}_0,\cdots, {\hat u}_d\}$ is linearly independent and spans the space $H_{\hat w}$. By
\begin{equation*}
\langle {\hat u}_j,{\hat u}_j\rangle_{H_w}=\langle \mathbf{u}_i,\mathbf{T}^*\mathbf{T}(\mathbf{u}_j)\rangle_{H_\mathbf{w}}=\langle \mathbf{u}_i,\mathbf{u}_j\rangle_{H_\mathbf{w}}=\delta_{i,j},
\end{equation*}
$\{{\hat u}_0,\cdots, {\hat u}_d\}$ is an o.n. basis of $H_{\hat w}$. Item 3 is proved.
\end{proof}

Denote by $H_0$ the orthogonal complement of $H_{\hat w}$ with respect to $H_w$: $H_w=H_{\hat w}\oplus H_0$, $H_0\perp H_{\hat w}$. Define $w_0=w-{\hat w}$. Then $w_0$ is the kernel of the RKHS $H_0$.
\begin{theorem}\label{th3.2}
The out-of-sample extension given by $\mathbf{T}$ from $H_\mathbf{w}$ to $H_{\tilde w}$ is exact if and only if $\dim(H_{\tilde w})=\dim(H_\mathbf{w})$, or equivalently, $H_0=\{0\}$.
\end{theorem}
\begin{proof}
Since the proof is similar to that for Theorem~\ref{th2.3}, we skip the details here.
\end{proof}

When we need to extend the DR data from $\mathbf{X}$ to $\mathcal{Z}$ using Diffusing Maps, we apply the multiplier $\mathfrak{S}$. We present the out-of-sample extension algorithm for Diffusion Maps in the following:
\begin{theorem}
Let the multiplier $\mathfrak{S}$ be defined by \eqref{eq4.0}. Then we have the following:
\begin{itemize}
\item The weighted out-of-sample DR extension for $x\in X$ is $\tilde{u}_j(x)=\mathbf{T}(\mathbf{v}_j)(x)$.
\item The standard out-of-sample DR extension for $x\in X$ is $\mathfrak{S}^{-1/2}(\tilde{u}_j(x))$.
\item The normalized out-of-sample DR extension is $\mathfrak{S}^{-1}(\tilde{u}_j(x))$.
\end{itemize}
\end{theorem}
The algorithm for these extensions is presented in \textbf{Algorithm 1}.
\begin{algorithm}[!ht]
  \caption{Diffusion Maps Out-of-Sample Extension Algorithm}
    \begin{algorithmic}[1]
    \Require
      Training data set $\mathbf{X}=[\mathbf{x}_1, \mathbf{x}_2, \cdots, \mathbf{x}_N]$;
      testing data set  $\mathbf{Z}=[\mathbf{z}_1, \mathbf{x}_2, \cdots, \mathbf{z}_M]$;
      kernel parameter $\epsilon$ for creating the kernel $W(\mathbf{x},\mathbf{y})=\exp(\|\mathbf{x}-\mathbf{y}\|^2/\epsilon)$; and optional threshold $\eta>0$ for constructing sparse weight matrix:
      \begin{equation*}
      \mathbf{w}(\mathbf{x},\mathbf{y})=\begin{cases} W(\mathbf{x},\mathbf{y}), & W(\mathbf{x},\mathbf{y})\ge \eta,\\ 0 , & W(\mathbf{x},\mathbf{y})< \eta.\end{cases}
      \end{equation*}
    \Ensure
      (1) Out-of-sample extension $U(\mathbf{Z})$ for the weighted DR of $\mathbf{Z}$.
      (2) Out-of-sample extension $\Psi(\mathbf{Z})$ for the standard DR of $\mathbf{Z}$.
      (3) Out-of-sample extension $V(\mathbf{Z})$ for the normalized DR of $\mathbf{Z}$.
    \State Part I. Make DR on the train data set $\mathbf{X}$.
    \State Create kernel $\mathbf{w}(\mathbf{x}_i,\mathbf{x}_j)$ on $\mathbf{X}^2$ using $\epsilon$ and $\eta$.
    \State Compute the density functions: $S_i=\sum_{j=1}^N \mathbf{w}(\mathbf{x}_i,\mathbf{x}_j)$ and the total mass $S=\sum_{i=1}^N S_i$.
    \State Construct the kernel for Diffusion Maps: Set $K=[ \mathbf{k}(\mathbf{x}_i,\mathbf{x}_j)]_{i,j=1}^N$, where  $\mathbf{k}(\mathbf{x}_i,\mathbf{x}_j)=\frac{\mathbf{w}(\mathbf{x}_i,\mathbf{x}_j)}{\sqrt{S_i S_j}}$.
    \State Make the spectral decomposition of $\mathbf{k}$, according to \eqref{eq3.0}:
    $K = \Psi\Lambda \Psi^T$, where $\Psi\in\mathbb{R}^{N\times d}$.
    \State Part II. Make out-of sample extension.
    \State Set $S=\diag(S_1,\cdots, S_N)$ and compute $V=S^{-1/2}\Psi\Lambda^{1/2}$.
    \State Compute $K_\mathbf{Z}=[\mathbf{w}(\mathbf{x}_i,\mathbf{z}_j)]_{i,j=1}^{N,M}$.
    \State Compute the weighted extension DR according to \eqref{eq3.6}: $U(\mathbf{Z})= K^T_\mathbf{Z} V \Lambda^{-1}$.
    \State Compute the updated density function: $\tilde{S}_i =  S_i + \sum_{j=1}^M \mathbf{w}(\mathbf{x}_i,\mathbf{z}_j)$ and set $\tilde{S}=\diag(\tilde{S}_1,\cdots,\tilde{S}_N)$.
    \State Compute the standard extension DR for $\mathcal{Z}$: $\Psi(\mathbf{Z})= \tilde{S}^{-1/2}U(\mathbf{Z})$.
    \State Compute the normalized extension DR for $\mathbf{Z}$: $V(\mathbf{Z})= \tilde{S}^{-1}U(\mathbf{Z})$.
\end{algorithmic}
\end{algorithm}

\newpage

\end{document}